\documentclass[letterpaper, 10 pt, conference]{ieeeconf}  %
\IEEEoverridecommandlockouts                              %
\usepackage{amsthm} 
\usepackage{amssymb}  
\usepackage{breqn}

\newtheorem{theorem}{Theorem}

\newtheorem{lemma}{Lemma}

\newtheorem{Assumption}{Assumption}
\usepackage{amsmath,amsfonts,amssymb,color}
\usepackage{mathtools}

\usepackage{epsfig}
\usepackage{algorithm}
\usepackage{algpseudocode}
\usepackage{dsfont}
\usepackage{multirow}

\usepackage{array}
\usepackage{multirow}
\usepackage{mathtools}

\DeclareMathOperator*{\argmin}{\arg\!\min}
\usepackage{epstopdf}
\usepackage{tikz}
\usepackage{relsize}
\usetikzlibrary{shapes,arrows}
\usepackage{cite}
\usepackage{epstopdf}
\usepackage{tikz}
\usetikzlibrary{shapes}
\usepackage[labelfont=normalfont]{caption}

\newcommand{\cmark}{\checkmark}%
\newcommand{\xmark}{$\times$}%
\definecolor{winered}{rgb}{0.5,0,0}
\usepackage{scalerel}
\usepackage{xcolor}
\usepackage[colorlinks]{hyperref}
\AtBeginDocument{%
  \hypersetup{
    citecolor=blue,
    linkcolor=winered}}

\title{\LARGE \bf
Min-Max Optimization under Delays
}

\author{Arman Adibi, Aritra Mitra, and Hamed Hassani
\thanks{A. Adibi and H. Hassani are with the Department of Electrical and Systems Engineering, University of Pennsylvania. Email: {\tt \{aadibi, hassani\}@seas.upenn.edu}. A. Mitra is with the Department of Electrical and Computer Engineering,  North Carolina State University. Email: {\tt amitra2@ncsu.edu}. This work was supported by NSF Award 1837253, NSF CAREER award CIF 1943064, and the Air Force Office
of Scientific Research Young Investigator Program (AFOSR-YIP) under award FA9550-20-1-0111.}%
}

\begin{document}

\maketitle
\thispagestyle{empty}
\pagestyle{empty}

\begin{abstract}
Delays and asynchrony are inevitable in large-scale machine-learning problems where communication plays a key role. As such, several works have extensively analyzed stochastic optimization with delayed gradients. However, as far as we are aware, no analogous theory is available for min-max optimization, a topic that has gained recent popularity due to applications in adversarial robustness, game theory, and reinforcement learning. Motivated by this gap, we examine the performance of standard min-max optimization algorithms with delayed gradient updates. First, we show (empirically) that even small delays can cause prominent algorithms like Extra-gradient (\texttt{EG}) to diverge on simple instances for which \texttt{EG} guarantees convergence in the absence of delays. Our empirical study thus suggests the need for a careful analysis of delayed versions of min-max optimization algorithms. Accordingly, under suitable technical assumptions, we prove that Gradient Descent-Ascent (\texttt{GDA}) and \texttt{EG} with delayed updates continue to guarantee convergence to saddle points for convex-concave and strongly convex-strongly concave settings. Our complexity bounds reveal, in a transparent manner, the slow-down in convergence caused by delays.
\end{abstract}
\section{Introduction}
Min-max optimization is a fundamental problem with applications in various fields, including game theory \cite{von}, 
 machine learning \cite{goodfellow2020generative}, robust optimization \cite{ben2009robust}, and more recently, adversarial robustness \cite{madry2017towards}. As such, the convergence analysis of various min-max optimization algorithms has received considerable attention over the years \cite{korpelevich, nedic, daskalakis, mokhtari2020unified}. While this has resulted in a rich literature that provides non-asymptotic guarantees for the vanilla versions of these algorithms, not much is known about their \textit{robustness} to different types of perturbations that show up in practice. In particular, for large-scale machine learning problems involving communication between multiple servers and agents, such perturbations get manifested in the form of (unavoidable) delays and asynchrony. Consequently, several works have extensively studied stochastic optimization with delayed gradients; since the literature on this topic is vast, we refer the reader to \cite{duchi2015asynchronous, doan2017, arjevani2020tight, stich19, koloskova2022} and the references therein. However, to our knowledge, there is no analogous theory for min-max optimization. Motivated by this gap, the goal of our paper is to build an understanding of the effect of delays on the convergence of common min-max optimization algorithms like Gradient Descent-Ascent (\texttt{GDA}) and Extra-Gradient (\texttt{EG}). Our main contributions in this regard are as follows.  

\subsection{Summary of Main Results}
$\bullet$ We start with a result that is perhaps surprising. In Section~\ref{sec:Divergence}, we empirically examine the effect of delays on the behavior of the Extra-Gradient 
 algorithm due to Korpelevich~\cite{korpelevich}. We observe that even with the smallest possible delay, i.e., a unit delay, \texttt{EG} diverges on a simple convex-concave function; see Fig.~\ref{fig:EG_diverge}.\footnote{The Gradient Descent-Ascent (\texttt{GDA}) algorithm diverges on this instance even in the absence of delays \cite{daskalakis}.\label{ftn:div}} Notably, in the absence of delays, \texttt{EG} provably guarantees convergence to a saddle-point for this function. This observation, although empirical, suggests that \textit{delays can have non-trivial effects on the convergence of popular min-max optimization algorithms}. 

$\bullet$ Our empirical study conveys the message that technical assumptions that are typically not required to study vanilla \texttt{EG} might, in fact, turn out to be needed to ensure convergence under delays. Accordingly, in Section~\ref{section:DEG_CC}, we study \texttt{DEG} - a version of \texttt{EG} with updates based on delayed gradients - for smooth, convex-concave functions over a \textit{bounded} domain. In Theorem~\ref{thm:DEGCC}, we show that \texttt{DEG} guarantees convergence to a saddle-point at a rate $O(\sqrt{\tau_{\max}}/{\sqrt{T}})$, where $T$ is the number of iterations, and $\tau_{\max}$ is a uniform bound on the delays. Our proof of this result is based on a connection to adversarial perturbations on statistical min-max learning problems in the recent work~\cite{adibi2022distributed}. 

In the absence of delays, the convergence rates of \texttt{EG} and Gradient Descent-Ascent (\texttt{GDA}) are $O(1/T)$ \cite{aryan} and $O(1/\sqrt{T})$ \cite{nedic}, respectively. Our empirical divergence result (see Footnote~\ref{ftn:div}) and Theorem~\ref{thm:DEGCC} collectively suggest \textit{that under delays, the behavior of \texttt{EG} is similar to that of \texttt{GDA}.}  

$\bullet$ To further investigate the above point, we turn our attention to the behavior of \texttt{GDA} under delays in Section~\ref{sec:DGDA_CC}; we refer to this delayed version as $\texttt{DGDA}$. For smooth, convex-concave functions with bounded gradients, we prove that \texttt{DGDA} exhibits a convergence rate of $O(\sqrt{\tau_{\max}}/{\sqrt{T}})$ - exactly like \texttt{DEG}; see Theorem~\ref{thm:DGDACC}. However, unlike the analysis for \texttt{DEG}, we do not assume a bounded domain. Instead, we provide a careful analysis to argue that with suitable step-sizes, the iterates of \texttt{DGDA} remain bounded. 

$\bullet$ All our results above pertain to scenarios where there is some underlying assumption of boundedness (either on the gradients or on the domain). Thus, one may ask: \textit{Can min-max optimization algorithms under delays converge in the absence of such boundedness assumptions?} In Section~\ref{sec:DGDA_SCSC}, we answer this question in the affirmative by studying \texttt{DGDA} for smooth, strongly convex-strongly concave functions. We prove that \texttt{DGDA} guarantees \textit{linear convergence to the saddle point} at a rate of $O(\exp(-T/\tau^3_{\max}))$; see Theorem~\ref{thm:scsc}. 

As far as we are aware, our results above are novel and provide the first steps toward theoretically understanding the robustness of min-max optimization algorithms to delay-induced perturbations. Our results are summarized in Table~\ref{tab:1}. 

\section{Problem Setting}
\begin{table*}[t!]
 \caption{The table below presents a summary of our findings, outlining the conditions required for each algorithm to achieve the specified convergence rate. In the smooth convex-concave case, the convergence rate corresponds to the number of iterations needed for the duality gap to be less than $\epsilon$. For the smooth strongly convex-strongly concave case (SC-SC), the rate corresponds to the number of iterations needed for the distance to saddle points to be less than $\epsilon$. It is worth noting that in this table, we hide the dependence on $G$, $L$, and the strong-convexity parameter in the ${O}$ notation.}
  \centering
  \begin{tabular}{c|ccccc}
   \hline
    Algorithm & Bounded Gradient & Bounded Domain & SC-SC & Convex-Concave & Convergence Rate \\
     \hline
     \hline\\
    Delayed Extra-Gradient (\texttt{DEG}) & \cmark & \cmark& \xmark&\cmark& $\mathcal{O}(\frac{\tau_{\max}}{\epsilon^2})$\\\\\hline
    \\
    Delayed Gradient Descent-Ascent (\texttt{DGDA}) & \cmark & \xmark& \xmark&\cmark& ${O}(\frac{\tau_{\max}}{\epsilon^2})$\\
    \\
     \hline\\
      Delayed Gradient Descent-Ascent (\texttt{DGDA}) & \xmark & \xmark& \cmark&\xmark & $ {O}(\tau_{\max}^3\log(\frac{1}{\epsilon}))$\\\\
     \hline  
  \end{tabular}
  \label{tab:1}
\end{table*}

In this section, we start by describing the basic setup of a min-max optimization problem. Next, we show empirically how \texttt{EG} can diverge with even one-step delays. Finally, we conclude the section by outlining some technical assumptions that will be made for the majority of the paper to ensure boundedness and convergence of iterates. 

\textbf{The basic min-max optimization setup.} Let $\mathcal{X}$ and $\mathcal{Y}$ be nonempty, convex subsets of $\mathbb{R}^m$ and $\mathbb{R}^n$, respectively.\footnote{While we will assume that $\mathcal{X}$ and $\mathcal{Y}$ are bounded sets in Section~\ref{section:DEG_CC}, this assumption will be later relaxed in Sections~\ref{sec:DGDA_CC} and \ref{sec:DGDA_SCSC}.} Given a mapping of the form $f:\mathcal{X} \times \mathcal{Y} \rightarrow \mathbb{R}$, we are interested in solving the following optimization problem:
\begin{equation}
\min_{x \in \mathcal{X}}\max_{y \in \mathcal{Y}} f(x,y).
\label{eq:opt}
\end{equation}
Throughout the paper, we will assume that $f(x,y)$ is \textit{continuously differentiable} in $x$ and $y$, and  \textit{convex-concave} over $\mathcal{X} \times \mathcal{Y}$. Specifically, $f(\cdot, y): \mathcal{X} \rightarrow \mathbb{R}$ is convex for every $y\in\mathcal{Y}$, and $f(x,\cdot): \mathcal{Y} \rightarrow \mathbb{R}$ is concave for every $x \in \mathcal{X}$. Our goal is to find a saddle point $(x^*,y^*)$ of $f(x,y)$ over the set $\mathcal{X} \times \mathcal{Y}$, where a saddle point is defined as a vector pair $(x^*,y^*) \in \mathcal{X} \times \mathcal{Y}$ that satisfies
\begin{equation}
    f(x^*,y) \leq f(x^*,y^*) \leq f(x,y^*), \forall x \in \mathcal{X}, y \in \mathcal{Y}.
\label{eqn:saddle_point}
\end{equation}

For any $\bar{x} \in \mathcal{X}$ and $\bar{y}\in\mathcal{Y}$, let  $\nabla_x f(\bar{x},\bar{y})$ and $\nabla_y f(\bar{x},\bar{y})$ denote the partial gradients of $f(x,y)$ with respect to $x$ and $y$, respectively, at $(\bar{x},\bar{y})$. Typical first-order iterative min-max optimization algorithms such as \texttt{GDA}, \texttt{EG}, and Optimistic Gradient Descent-Ascent (\texttt{OGDA}) aim to solve for $(x^*, y^*)$ based on an oracle that provides partial gradients of $f(x,y)$ evaluated at the most recent iterates of the algorithm. 

\textbf{The delay model.} Not much, however, is known about scenarios where the oracle is \textit{imperfect}. To that end, we studied the effect of adversarial perturbations on the partial gradients of $f(x,y)$ in our recent work~\cite{adibi2022distributed}. In this work, we take a different stance. Instead of considering \textit{arbitrary} adversarial perturbations, we will focus on \textit{structured} perturbations induced by delays. As mentioned earlier in the Introduction, the source of such delays could be communication latencies or system-level computational challenges such as stragglers, both of which are prevalent in distributed systems. 
In this work, given an iterative min-max optimization algorithm that generates a sequence of iterates $\{(x_k, y_k)\}$, we assume that at iteration $k$, we only have access to partial gradients of $f(x,y)$ computed at a \textit{stale} iterate $(x_{k-\tau_k}, y_{k-\tau_k})$, i.e., we have access to $\nabla_x f(x_{k-\tau_k}, y_{k-\tau_k})$ and $\nabla_y f(x_{k-\tau_k}, y_{k-\tau_k})$, where $\tau_k$ is the delay at iteration $k$. While we allow the delays to be time-varying, throughout the paper, we will work under the running assumption that all delays are uniformly bounded, i.e., there exists some positive integer $\tau_{\max}$ such that  $\tau_k \leq \tau_{\max}, \forall k$. 

Our goal is to understand 
what happens, when for computing the next iterate $(x_{k+1},y_{k+1})$, one uses these delayed gradients as opposed to $\nabla_x f(x_{k}, y_{k})$ and $\nabla_y f(x_{k}, y_{k})$. Specifically, we ask:
\begin{itemize}
    \item Can we hope for convergence to saddle points using delayed versions of algorithms like \texttt{GDA} and \texttt{EG}? 

    \item If so, for different classes of functions,  how do the convergence rates get affected by $\tau_{\max}$?
\end{itemize}

In the next subsection, we demonstrate (empirically) that the answers to such questions are more nuanced than what one might initially expect. 

\subsection{Divergence of Extra-Gradient Algorithm under Delay}
\label{sec:Divergence}
Let us start by quickly reviewing how the Extra-gradient (\texttt{EG}) algorithm for finding saddle-points operates in an unconstrained setting. \texttt{EG} first computes a set of mid-points $(\hat{x}_k, \hat{y}_k)$ by using partial gradients evaluated at the current iterate $(x_k, y_k)$: 
\begin{equation}
\begin{aligned}
    \hat{x}_k &\gets x_k-\alpha \nabla_x f(x_{k},y_{k})  \\
    \hat{y}_k &\gets y_k+\alpha \nabla_y f(x_{k},y_{k}),
\end{aligned}
\label{eqn:mid_point_d21}
\end{equation}
where $\alpha$ is a suitable step-size. Next, using gradients evaluated at the mid-points, \texttt{EG} computes the next iterates as 
\begin{equation}
\begin{aligned}
    {x}_{k+1} &\gets x_k-\alpha \nabla_x f(\hat x_{k},\hat y_{k}) \\
    {y}_{k+1} &\gets y_k+\alpha \nabla_x f(\hat x_{k},\hat y_{k}). 
\end{aligned}
\label{eqn:end_point_d22}
\end{equation}

For smooth, convex-concave functions, the above \texttt{EG} procedure guarantees convergence to a saddle-point at a rate of $O(1/T)$, where $T$ is the number of iterations \cite{aryan}. Moreover, to achieve this convergence, one does not need to make any assumption of a bounded domain or bounded gradients. 

Now to illustrate the challenges posed by delays, let us consider solving the following problem
\begin{equation}
\min_x\max_y \langle x,y\rangle,
\label{eqn:ex}
\end{equation}
using a version of \texttt{EG} where all partial gradients are evaluated at iterates that are delayed by just one time-step.\footnote{Formally, the delayed \texttt{EG} algorithm we study is outlined in Algorithm~\ref{algo:DEG}.} Whereas one might have expected a slow-down in convergence due to delays, Figure \ref{fig:EG_diverge} shows that in this specific setting, a unit delay causes \texttt{EG} to diverge! This demonstrates that delays can lead to non-trivial phenomena for standard min-max algorithms, thereby justifying our current study. 

A rough explanation for the above phenomenon is as follows. In \cite{mokhtari2020unified}, the authors argued that \texttt{EG} can be studied as an approximate version of the Proximal Point (PP) algorithm, which, in turn, operates as follows:
\begin{equation}
\begin{aligned}
{x}_{k+1} &\gets x_k-\alpha \nabla_x f( x_{k+1}, y_{k+1})  \\
{y}_{k+1} &\gets y_k-\alpha \nabla_y f( x_{k+1}, y_{k+1}).
\end{aligned}
\end{equation}
\begin{figure}
    \centering
\includegraphics[width=9cm]{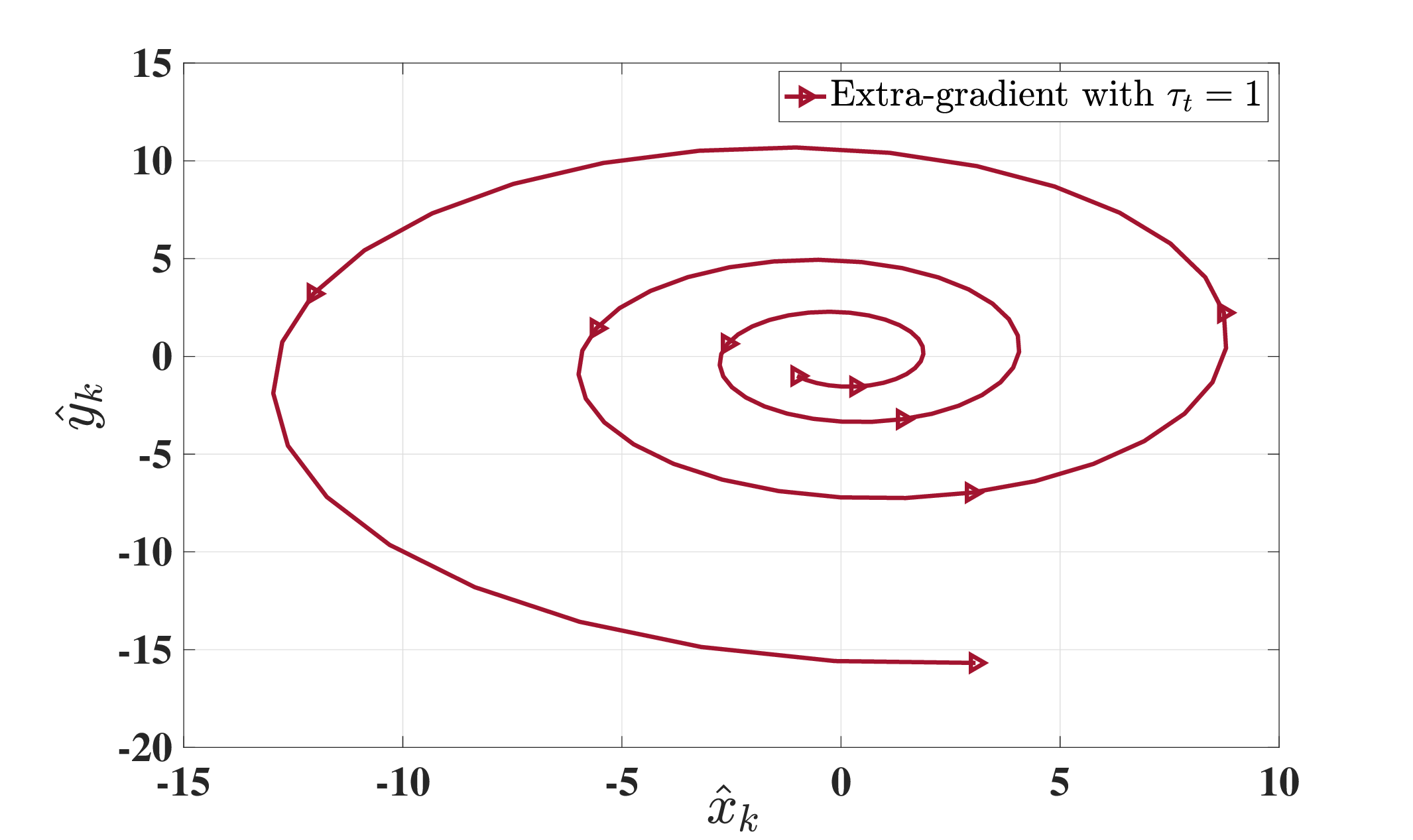}
    \caption{The Extra-gradient algorithm fails to converge, even with just one step delay, for the optimization problem $\min_x\max_y \langle x, y\rangle$. In this plot, we used a step size of $\alpha=0.2$. However, with the same step size and no delay, the Extra-gradient algorithm converges to the origin, which is the saddle-point for this problem.}
    \label{fig:EG_diverge}
\end{figure}

When the gradients on the right-hand side of the above equations are evaluated at one-step-delayed iterates, the above algorithm reduces to the \texttt{GDA} algorithm. Unlike \texttt{EG}, however, \texttt{GDA} can diverge for smooth, convex-concave problems like the one in Eq.~\eqref{eqn:ex}, \textit{even in the absence of delays}. In particular, some assumption on the boundedness of domain or gradients is needed to ensure the convergence of \texttt{GDA} for convex-concave problems. From the above discussion, we conclude that since \texttt{EG} with delays tends to behave like \texttt{GDA}, we need to impose additional technical assumptions to ensure convergence to saddle points. As such, we will impose the following assumption of bounded gradients at various points in the paper. 

\begin{Assumption}
\label{ass:boundedgrad} There exists a constant $G > 1$ such that the following holds for all $x \in \mathcal{X}$, and all $y \in \mathcal{Y}$: $\Vert \nabla_x f(x,y) \Vert \leq G$, and $ \Vert \nabla_y f(x,y) \Vert \leq G$.\footnote{We will use $\Vert \cdot \Vert$ to represent the Euclidean norm.} 
\end{Assumption}

We will also make the following standard assumption that the partial gradients of $f(x,y)$ are Lipschitz continuous. 

\begin{Assumption}
\label{ass:smoothness} There exists a constant $L > 1$ such that the following holds for all $x_1, x_2 \in \mathcal{X}$, and all $y_1, y_2 \in \mathcal{Y}$:
\begin{equation}
\begin{aligned}
    \Vert \nabla_x f(x_1,y_1) - \nabla_x f(x_2,y_2) \Vert &\leq L \left(\Vert x_1-x_2 \Vert + \Vert y_1 - y_2 \Vert \right), \\
    \Vert \nabla_y f(x_1,y_1) - \nabla_y f(x_2,y_2) \Vert &\leq L \left(\Vert x_1-x_2 \Vert + \Vert y_1 - y_2 \Vert \right).
\end{aligned}  
\nonumber
\end{equation}
\end{Assumption}
\section{Analysis of Delayed Extra-gradient for Convex-Concave functions}
\label{section:DEG_CC}
In Section \ref{sec:Divergence}, we saw that in the absence of a projection step to ensure the boundedness of iterates, the \texttt{EG} algorithm diverges on very simple functions, even with a one-step delay. Based on this empirical observation, in this section, we study delayed extra-gradient (\texttt{DEG}) - outlined in  Algorithm~\ref{algo:DEG} - under additional assumptions. In particular, throughout this section, we will work under  Assumptions~\ref{ass:boundedgrad} and \ref{ass:smoothness}, i.e., we will assume that the partial gradients of $f(x,y)$ are Lipschitz continuous and uniformly bounded.  It is important to note here that the divergence of \texttt{DEG}, as illustrated in Figure \ref{fig:EG_diverge}, occurs when we do not impose Assumption \ref{ass:boundedgrad}. Thus, this assumption will play a crucial role in our analysis. 

The update rule for \texttt{DEG} (Algorithm~\ref{algo:DEG}) involves two steps. In the first step, \texttt{DEG} computes a midpoint $(\hat{x}_k, \hat{y}_k)$ based on partial gradients evaluated at a stale iterate $(x_{k-\tau_k}, y_{k-\tau_k})$; see Eq.~\eqref{eqn:mid_point_d}. In the second step, \texttt{DEG} computes the next iterate  $(x_{k+1},y_{k+1})$ based on partial gradients evaluated at a stale mid-point $(\hat{x}_{k-\hat{\tau}_k}, \hat{y}_{k-\hat{\tau}_k})$; see Eq.~\eqref{eqn:end_point_d2}. There are two important things to take note of here. First, in each of the above updates, we project onto $\mathcal{X} \times \mathcal{Y}$ to ensure the boundedness of iterates. Second, our analysis is general enough to accommodate time-varying delays; furthermore, we allow  ${\tau}_k$ and $\hat{\tau}_k$ to also be different. That said, as mentioned before, we will work under the running assumption that all delays are bounded uniformly by $\tau_{\max}$, i.e., $\max\{{\tau}_k, \hat{\tau}_k\} \leq \tau_{\max}, \forall k$.  

\textbf{Key Insight and Outline of Analysis.} The starting point of our analysis for \texttt{DEG} is the observation that the errors induced by delays can be interpreted as \textit{bounded perturbations}. As we shall see in Lemma \ref{lemma:error}, the boundedness of the delay-induced errors follows as a direct consequence of Assumptions \ref{ass:boundedgrad} and \ref{ass:smoothness}, and the uniform boundedness assumption on the delays. This key observation allows us to immediately make a connection to our prior work in \cite{adibi2022distributed}, where we studied min-max optimization under adversarial perturbations. Building on this connection, we start with the following result that establishes some basic inequalities for our subsequent analysis; the proof of this result follows the same steps as that of \cite[Lemma 1]{adibi2022distributed}. 

\begin{algorithm}[t]
\caption{Delayed Extra-Gradient  (\texttt{DEG})} 
\label{algo:DEG}
\begin{algorithmic}[1]
\Require Initial vectors $x_1 \in \mathcal{X}$, $y_1 \in \mathcal{Y}$; algorithm parameters: step-size $\alpha > 0$.
\For {$k=1,\ldots,T$} 

\State  
\begin{equation}
\begin{aligned}
    \hat{x}_k &\gets \Pi_{\mathcal{X}}\left(x_k-\alpha \nabla_x f(x_{k-\tau_k},y_{k-\tau_k}) \right) \\
    \hat{y}_k &\gets \Pi_{\mathcal{Y}}\left(y_k+\alpha \nabla_y f(x_{k-\tau_k},y_{k-\tau_k}) \right).
\end{aligned}
\label{eqn:mid_point_d}
\end{equation}
\State 

\begin{equation}
\begin{aligned}
    {x}_{k+1} &\gets \Pi_{\mathcal{X}}\left(x_k-\alpha \nabla_x f(\hat x_{k-\hat\tau_k},\hat y_{k-\hat\tau_k}) \right) \\
    {y}_{k+1} &\gets \Pi_{\mathcal{Y}}\left(y_k+\alpha \nabla_x f(\hat x_{k-\hat\tau_k},\hat y_{k-\hat\tau_k})\right).
\end{aligned}
\label{eqn:end_point_d2}
\end{equation}
\EndFor
\end{algorithmic}
\end{algorithm} 
\label{sec:Proofs}

\begin{lemma} 
\label{lemma:basic}
For the \texttt{DEG} algorithm, the following inequalities hold for all $k\in [T], x\in \mathcal{X}$, and $y\in\mathcal{Y}$:\footnote{Given a positive integer $N$, we use $[N]$ to represent the set $\{1, \ldots, N\}.$}
\begin{equation}
\resizebox{1\hsize}{!}{$
    \begin{aligned}
    2\alpha \langle \nabla_x f(x_{k-\tau_k},y_{k-\tau_k}), \hat{x}_k-x \rangle &\leq {\Vert x-x_k \Vert}^2 - {\Vert x-\hat{x}_k \Vert}^2 - {\Vert \hat{x}_k-x_k \Vert}^2 \\ 
    -2\alpha \langle \nabla_y f(x_{k-\tau_k},y_{k-\tau_k}), \hat{y}_k-y \rangle &\leq {\Vert y-y_k \Vert}^2 - {\Vert y-\hat{y}_k \Vert}^2 - {\Vert \hat{y}_k-y_k \Vert}^2 \\ 
    2\alpha \langle \nabla_x f(\hat x_{k-\hat\tau_k},\hat y_{k-\hat\tau_k}), {x}_{k+1}-x \rangle &\leq {\Vert x-x_k \Vert}^2 - {\Vert x-{x}_{k+1} \Vert}^2 - {\Vert {x}_{k+1}-x_k \Vert}^2 \\ 
   - 2\alpha \langle \nabla_y f(\hat x_{k-\hat\tau_k},\hat y_{k-\hat\tau_k}), {y}_{k+1}-y \rangle &\leq {\Vert y-y_k \Vert}^2 - {\Vert y-{y}_{k+1} \Vert}^2 - {\Vert {y}_{k+1}-y_k \Vert}^2.
    \end{aligned}
$}
\nonumber
\label{eqn:basic_relations}
\end{equation}
\end{lemma}

Next, to bound the impact of delays, we introduce the following error vectors:
\begin{align*}
    e_x({x}_k,{y}_k) \triangleq \nabla_x f( x_{k-\tau_k}, y_{k-\tau_k})-\nabla_x f({x}_k,{y}_k),\\  e_y({x}_k,{y}_k) \triangleq \nabla_y f( x_{k-\tau_k},y_{k-\tau_k})-\nabla_y f({x}_k,{y}_k),
\end{align*}
and
\begin{align*}
    e_x(\hat{x}_k,\hat{y}_k) \triangleq \nabla_x f(\hat x_{k-\hat\tau_k},\hat y_{k-\hat\tau_k})-\nabla_x f(\hat{x}_k,\hat{y}_k),\\  e_y(\hat{x}_k,\hat{y}_k) \triangleq \nabla_y f(\hat x_{k-\hat\tau_k},\hat y_{k-\hat\tau_k})-\nabla_y f(\hat{x}_k,\hat{y}_k).
\end{align*}

Let $D=\max\{D_x,D_y\}$, where $D_x$ and $D_y$ are the diameters of the sets $\mathcal{X}$ and $\mathcal{Y}$, respectively. Leveraging Lemma \ref{lemma:basic}, our next result tracks the progress made by the mid-point sequence $\{(\hat{x}_k, \hat{y}_k)\}$ generated by \texttt{DEG}. The proof of this result mirrors that of \cite[Lemma 2]{adibi2022distributed}. 
 
\begin{lemma} 
\label{lemma:midpoints} Suppose Assumptions \ref{ass:boundedgrad} and \ref{ass:smoothness} hold. Furthermore, suppose $\alpha \leq 1/(2L)$. Then, for the \texttt{DEG} algorithm, the following holds for all $k\in [T], x\in \mathcal{X}$, and $y\in \mathcal{Y}$:
\begin{equation}
\resizebox{1\hsize}{!}{$
    \begin{aligned}
    & \alpha \langle \nabla_x f(\hat{x}_k, \hat{y}_k), \hat{x}_k - x \rangle - \alpha \langle \nabla_y f(\hat{x}_k, \hat{y}_k), \hat{y}_k - y \rangle \\
    & \leq \frac{1}{2} \left( {\Vert x-x_k \Vert}^2 - {\Vert x-x_{k+1} \Vert}^2 + {\Vert y-y_k \Vert}^2 - {\Vert y-y_{k+1} \Vert}^2\right)\\
    & + \alpha D \left(\Vert e_x(x_k, y_k) \Vert + \Vert e_x(\hat{x}_k, \hat{y}_k) \Vert + \Vert e_y(x_k, y_k) \Vert + \Vert e_y(\hat{x}_k, \hat{y}_k) \Vert \right).
    \end{aligned}
\nonumber
$}
\label{eqn:mid_point_bnd}
\end{equation}
\end{lemma}

The above result sets things up nicely for a telescoping-sum analysis. However, the missing piece right now is to provide bounds on the delay-induced errors. We derive such bounds in the following result. 

\begin{lemma} \label{lemma:error} Suppose Assumptions \ref{ass:boundedgrad} and \ref{ass:smoothness} hold. For the \texttt{DEG} algorithm, the following error-bounds then apply $\forall k \in [T]$: 
\begin{equation}
\resizebox{1\hsize}{!}{$
 \max\{\Vert e_x(x_k,y_k) \Vert, \Vert e_x(\hat{x}_k,\hat{y}_k) \Vert, \Vert e_y(x_k,y_k) \Vert, \Vert e_y(\hat{x}_k,\hat{y}_k) \Vert\} \leq \Delta_T, 
 $}
 \nonumber
 \end{equation}
where $\Delta_T =6\alpha GL \tau_{\max}$. 
\end{lemma}
\begin{proof}
In what follows, we only show how to bound $\|e_x(x_k,y_k)\|$ and $\|e_x(\hat x_k,\hat y_k) \|$; bounds for the other two error terms can be derived in an identical manner. We start by bounding $\|e_x(x_k,y_k)\|$. From equation \eqref{eqn:end_point_d2}, we have 
\begin{equation}
\begin{aligned}
\|x_k-x_{k-\tau_k}\|&\leq \sum_{j=k-\tau_k}^{k-1} \Vert x_{j+1} - x_{j} \Vert\\ 
& \overset{(a)}\leq \alpha \left(\sum_{j=k-\tau_k}^{k-1}\|\nabla_x f(\hat x_{j-\hat\tau_j},\hat y_{j-\hat\tau_j})\|\right) \\& \overset{(b)} \leq \alpha G\tau_{\max},\label{eq:dif time}
\end{aligned}
\end{equation}
where (a) follows from the non-expansive property of the projection operator, and (b) follows from Assumption~\ref{ass:boundedgrad} and the fact that $\tau_k \leq \tau_{\max}$. Using the exact same steps, one can establish the same bound for $\|y_k-y_{k-\tau_k}\|$. Thus, we have
\begin{equation}
\begin{aligned}
\|e_x(x_k,y_k)\|&=\|\nabla_x f(x_k,y_k)-\nabla_x f(x_{k-\tau},y_{k-\tau})\|\\&\overset{(a)}\leq L
(\|x_k-x_{k-\tau_k}\|+\|y_k-y_{k-\tau_k}\|)\\&\overset{(b)}\leq 2 \alpha GL{\tau_{\max}},\label{eq:bound ex}
\end{aligned}
\end{equation}
where (a) follows from smoothness, i.e., Assumption~\ref{ass:smoothness}, and (b) follows from Eq.~\eqref{eq:dif time}. Now to bound $e_x(\hat x_k,\hat y_k)$, observe  
\begin{equation}
\resizebox{1\hsize}{!}{$
\begin{aligned}
\|e_x(\hat x_k,\hat y_k)\|&=\|\nabla_x f(\hat x_{k-\hat\tau_k},\hat y_{k-\hat\tau_k})-\nabla_x f(\hat{x}_k,\hat{y}_k)\|\\& \overset{(a)} 
 \leq L
(\|\hat x_k-\hat x_{k-\hat\tau_k}\|+\|\hat y_k-\hat y_{k-\hat\tau_k}\|)\\&\leq L
(\|\hat x_k-x_k\|+\|x_{k}- x_{k-\hat\tau_k}\| + \|x_{k-\hat\tau_k}-\hat x_{k-\hat\tau_k}\| 
\notag\\&+\|\hat y_k-y_k\|+\| y_k - y_{k-\hat\tau_k}\|)+\|y_{k-\hat\tau_k}-\hat y_{k-\hat\tau_k}\|)\\& \overset{(b)} \leq 2\alpha GL{\tau_{\max}}+4\alpha GL
\\&\overset{(c)} \leq 6 \alpha GL{\tau_{\max}}.
\end{aligned}
$}
\end{equation}
In the above steps, (a) follows from Assumption \ref{ass:smoothness}, (b) follows from \eqref{eq:dif time} and the fact that for any $j \in [T]$, $\|\hat x_j-x_j\|\leq \alpha\|\nabla_x f(x_{j-\tau_j},y_{j-\tau_j})\|\leq \alpha G $, and (c) follows from noting that $\tau_{\max} \geq 1$. This concludes the proof. 
\end{proof}

We are now in a position to prove our first main result which establishes complexity bounds for \texttt{DEG} for smooth convex-concave functions with bounded gradients. 

\begin{theorem}
\label{thm:DEGCC}
    Suppose Assumptions  \ref{ass:boundedgrad} and \ref{ass:smoothness} hold. Moreover, suppose the number of iterations $T$ is large enough such that $T \geq L$. Then, with 
    $$\alpha=\sqrt{\frac{1}{24GL \tau_{\max} T}         },$$
    the iterates generated by \texttt{DEG} satisfy: 
\begin{align}
\max_{y \in \mathcal{Y}} f(\bar{x}_T,y)- \min_{x \in \mathcal{X}} f(x, \bar{y}_T) \leq 10 D^{2      } \sqrt{\frac{GL\tau_{\max}}{T}}, 
\label{eqn:duality}
\end{align}
where $\bar{x}_T=(1/T)\sum_{k\in[T]} \hat{x}_k$ and $\bar{y}_T=(1/T)\sum_{k\in[T]} \hat{y}_k$.
\end{theorem}
\begin{proof}
Let us start by noting that when $T \geq L$, the choice of step-size above satisfies $\alpha \leq 1/(2L)$. Thus, we can invoke Lemma~\ref{lemma:midpoints}. From the convex-concave property of $f(x,y)$, the following inequalities hold $\forall k\in[T], x\in \mathcal{X}$, and $y\in\mathcal{Y}$:
\begin{equation}
    \begin{aligned}
    \alpha \left( f(\hat{x}_k,\hat{y}_k)-f(x,\hat{y}_k) \right) & \leq \alpha \langle \nabla_x f(\hat{x}_k, \hat{y}_k), \hat{x}_k - x \rangle\\
   - \alpha \left( f(\hat{x}_k,\hat{y}_k)-f(\hat{x}_k,y) \right) & \leq - \alpha \langle \nabla_y f(\hat{x}_k, \hat{y}_k), \hat{y}_k - y \rangle. 
    \end{aligned}
\nonumber
\end{equation}
Summing the two inequalities above, and using Lemmas \ref{lemma:midpoints} and \ref{lemma:error}, we obtain:
\begin{equation}
    \begin{split}
\alpha \left( f(\hat{x}_k,{y})-f(x,\hat{y}_k) \right) \leq \frac{1}{2} \left( {\Vert x-x_k \Vert}^2 - {\Vert x-x_{k+1} \Vert}^2 \right) 
\\ + \frac{1}{2} \left( {\Vert y-y_k \Vert}^2 - {\Vert y-y_{k+1} \Vert}^2\right) + 4\alpha D \Delta_T,
    \end{split}
\label{eqn:interimbnd}
\end{equation}
where $\Delta_T$ is as defined in Lemma~\ref{lemma:error}. 
From the convexity of $f(x,y)$ w.r.t. $x$ and concavity w.r.t. $y$, note that we have $f(\bar{x}_T,y) \leq (1/T) \sum_{k\in[T]} f(\hat{x}_k,y)$ and  $f(x,\bar{y}_T) \geq  (1/T) \sum_{k\in[T]} f(x,\hat{y}_k)$, respectively. Combining this with Eq.~\eqref{eqn:interimbnd}, we obtain
\begin{equation}
f(\bar{x}_T,y)-f(x, \bar{y}_T) \leq \frac{D^2}{\alpha T} + 4 D \Delta_T.
\end{equation}
The result follows by plugging into the above inequality the choice of $\alpha$ in the statement of the theorem, and by noting that the resulting bound holds for all $x\in \mathcal{X}$ and for all $y \in \mathcal{Y}$. This completes the proof. 
\end{proof}

\textbf{Discussion.} From Theorem~\ref{thm:DEGCC}, we conclude that for smooth convex-concave functions, \texttt{DEG} guarantees that the primal-dual gap converges to zero at a rate $O(\sqrt{\tau_{\max}}/{\sqrt{T}})$. The primal-dual gap is zero if and only if $(\bar{x}_T,\bar{y}_T)$ is a saddle point of $f(x,y)$ over the set $\mathcal{X} \times \mathcal{Y}$.  Thus, \texttt{DEG} also guarantees convergence to a saddle-point under delays. The important thing to note here is that the $O(1/T)$ convergence rate of \texttt{EG} gets significantly slackened in the presence of delays; whether this is an artifact of our analysis or fundamental is an open question. The $O(1/\sqrt{T})$ rate of \texttt{DEG} mirrors the rate of \texttt{GDA} in the absence of delays. In the following sections, we will further explore this connection. 

\section{Analysis of Delayed Gradient Descent-Ascent for Convex-Concave functions}
\label{sec:DGDA_CC}
\begin{algorithm}[t]
\caption{Delayed Gradient Descent-Ascent  (\texttt{DGDA})} 

\label{algo:DGDA}
\begin{algorithmic}[1]
\Require Initial vector $z_1=[x_1;y_1] \in \mathbb{R}^{m+n}$; algorithm parameters: step-size $\alpha > 0$.
\For {$k=1,\ldots,T$} 
\State\begin{equation}
\begin{aligned}
   z_{k+1}=z_{k}-\alpha\Phi(z_{k-\tau_k}). 
\end{aligned}
\label{eqn:GDA}
\end{equation}
\EndFor
\end{algorithmic}
\end{algorithm} 
In this section, we will examine the convergence of a delayed version of the gradient descent ascent algorithm that we refer to as \texttt{DGDA}. As before, we will continue to work under Assumptions~\ref{ass:boundedgrad} and \ref{ass:smoothness}. However, we will set $\mathcal{X}=\mathbb{R}^m$ and $\mathcal{Y}=\mathbb{R}^n$, i.e., as a departure from the previous section, the domains of the variables $x$ and $y$ are no longer assumed to be bounded. As we shall soon see, this makes the analysis more challenging relative to that in Section~\ref{section:DEG_CC}. 

To proceed, given any $x \in \mathbb{R}^m$ and $y \in \mathbb{R}^n$, we will find it convenient to define a new variable $z = [x;y]$ that resides in $\mathbb{R}^{m+n}$. Next, corresponding to any $z = [x;y]$, let us define the function $\Phi: \mathbb{R}^{m+n} \rightarrow \mathbb{R}^{m+n}$ as follows:
\begin{equation}
\Phi(z)=\begin{bmatrix}
\nabla_x f(x,y)\\
-\nabla_y f(x,y)
\end{bmatrix},
\label{eq:Phi}
\end{equation}
With these notations in place, we outline the steps of \texttt{DGDA} in Algorithm~\ref{algo:DGDA}; the steps are self-explanatory.

\textbf{Analysis of \texttt{DGDA}}. In our analysis, we will make use of the following result from \cite{nemirov}, stated for our purpose. 

\begin{lemma}\label{lemma:nemirov} Let $\Phi(z)$ be as defined in Eq.~\eqref{eq:Phi}, and suppose Assumption \ref{ass:smoothness} holds for all $z\in \mathbb{R}^{m+n}$. Then, the following statements are true for any $z_1,z_2\in \mathbb{R}^{m+n}$:
    \begin{enumerate}
\item  $\langle\Phi(z_1)-\Phi(z_2),z_1-z_2\rangle \geq 0,$
\item  For any saddle-point $z^*=[x^*; y^*]$ of $f(x,y)$, we have $\Phi(z^*)=0$. 
    \end{enumerate}
\end{lemma}

We start with a simple result that bounds the error $e_k \triangleq \Phi(z_k)-\Phi(z_{k-\tau_k})$ induced by delays as a function of the smoothness parameter $L$, the uniform bound on the gradients $G$, and the maximum delay $\tau_{\max}$. 

\begin{lemma}
\label{lemma:DGDAerr}
Suppose Assumptions \ref{ass:boundedgrad} and \ref{ass:smoothness} hold $ \forall z \in \mathbb{R}^{m+n}$. Then, for any $k\in [T]$, the delay-induced error $e_k \triangleq \Phi(z_k)-\Phi(z_{k-\tau_k})$ for \texttt{DGDA} satisfies
\begin{equation}
    \|e_k\|\leq 2 \alpha  L G \tau_{\max}. 
\end{equation}
\end{lemma}
\begin{proof}
For any two points $z=[x;y]$ and $ \hat z=[\hat x; \hat y]$, we have
\begin{equation}
\begin{aligned}
    \|\Phi(z)-\Phi(\hat z)\|^2&\leq \overset{(a)} 2 (L(\|x-\hat x\|+\|y-\hat y\|))^2
    \\& \leq 4L^2\|z-\hat z\|^2,
    \label{eq:smooth phi}
\end{aligned}
\end{equation}
where we used Assumption \ref{ass:smoothness} for the first inequality. Based on the above inequality, we have 
\begin{equation}
\begin{aligned}
    \|e_k\|&=\|\Phi(z_{k-\tau_k}) -\Phi(z_{k})\|\\&\leq 2L\|z_{k-\tau_k} -z_{k}\|\\&\leq 2 L\sum_{j=k-\tau_k}^{k-1}\|z_{j+1} -z_{j}\|\\
    &\leq 2 \alpha L\sum_{j=k-\tau_k}^{k-1}\|\Phi(z_{j-\tau_j})\| 
    \leq 2 \alpha L G \tau_{\max},
\end{aligned}
\end{equation}
where the final step follows from Assumption \ref{ass:boundedgrad}.
\end{proof}

Unlike the analysis in Section~\ref{section:DEG_CC} where the boundedness of the domain implied bounded iterates, we need to do more work to establish that the iterates generated by \texttt{DGDA} remain bounded. Leveraging Lemma~\ref{lemma:DGDAerr}, the following result establishes this key fact.  

\begin{lemma}\label{lemma: main bound} Suppose Assumptions \ref{ass:boundedgrad} and \ref{ass:smoothness} hold $ \forall z \in \mathbb{R}^{m+n}$. Let $z^*=[x^*; y^*]$, and suppose the step-size $\alpha$ satisfies 
$$\alpha \leq \frac{1}{2 \sqrt{LG \tau_{\max} T}}.$$
Then, for the \texttt{DGDA} algorithm, the following holds $ \forall k\in [T]$: 
\begin{equation}
    \Vert z_k - z^* \Vert^2 \leq 10 B,  \hspace{1mm} \textrm{where} \hspace{1mm} B=\max\{\Vert z_1-z^* \Vert^2, G \}.
\end{equation}
\end{lemma}
\begin{proof}
From Eq.~\eqref{eqn:GDA} and the definition of $e_k$, we have
\begin{equation}
\resizebox{1\hsize}{!}{$
\begin{aligned}
    \|z_{k+1}-z\|^2=&\|z_{k}-\alpha\Phi(z_{k})-z\|^2\\
    &+\alpha^2\|e_k\|^2+2\alpha\langle e_k,z_k-z-\alpha\Phi(z_{k})\rangle 
    \\=&\|z_{k}-z\|^2+\alpha^2\|\Phi(z_{k})\|^2-2\alpha\langle\Phi(z_{k}),z_{k}-z\rangle\\
    &+\alpha^2\|e_k\|^2+2\alpha\langle e_k,z_k-z-\alpha\Phi(z_{k})\rangle
    \\\leq&\|z_{k}-z\|^2+2\alpha^2G^2-2\alpha\langle\Phi(z_{k}),z_{k}-z\rangle\\  &+4\alpha^4G^2L^2\tau_{\max}^2 \underbrace{+2\alpha\langle e_k,z_k-z\rangle}_{T_1}\underbrace{-2\alpha^2\langle e_k,\Phi(z_{k})\rangle}_{T_2}.
\end{aligned}
$}
\label{eqn:error_decomp}
\end{equation}
We now proceed to bound $T_1$ and $T_2$. For $T_2$, we have: 
\begin{equation}
\begin{aligned}
    T_2 \overset{(a)} \leq & \alpha^2 \Vert e_k \Vert^2 + \alpha^2 \Vert \Phi(z_{k}) \Vert^2\\
   \overset{(b)} \leq & 4\alpha^4G^2L^2\tau_{\max}^2 + 2\alpha^2 G^2,
\end{aligned}
\nonumber
\end{equation}
where (a) follows from the elementary fact that for any two scalars $c, d \in \mathbb{R}$, it holds that
\begin{equation}
    cd \leq \frac{1}{2} c^2 + \frac{1}{2} d^2.
\label{eqn:basic_ineq}
\end{equation}
Moreover, for (b),  we used Lemma~\ref{lemma:DGDAerr} 
and Assumption \ref{ass:boundedgrad}. For bounding $T_1$, observe that 
\begin{equation}
\begin{aligned}
    T_1&=2\alpha\langle e_k,z_k-z\rangle\\&\leq 2\alpha \|e_k\|\|z_k-z\|\\&  \overset{(a)} \leq 4\alpha^2  G L \tau_{\max}\|z_k-z\|\\
    &=\left(2\alpha \sqrt{GL \tau_{\max}} \right) \left(2\alpha \sqrt{GL \tau_{\max}}  \Vert z_k - z \Vert \right)\\
    & \overset{(b)} \leq 2 \alpha^2 GL \tau_{\max} + 2 \alpha^2 GL \tau_{\max} \Vert z_k - z \Vert^2,
\end{aligned}
\end{equation}
where we again appealed to Lemma~\ref{lemma:DGDAerr} for (a). For (b), we used Eq.~\eqref{eqn:basic_ineq}. Plugging in the above bounds on $T_1$ and $T_2$ into Eq.~\eqref{eqn:error_decomp}, simplifying using $L, G \geq 1$, and rearranging terms, we arrive at the following inequality: 
\begin{equation}
\begin{aligned} 
2\alpha\langle\Phi(z_{k}),z_{k}-z\rangle  
  & \leq \left(1+2\alpha^2 L G \tau_{\max}\right) \|z_{k}-z\|^2\\ 
  & \hspace{2mm} - \|z_{k+1}-z\|^2+ A,
\end{aligned}
\label{eqn:key_ineq}
\end{equation}
where $A=2 \alpha^2 G L \tau_{\max} (1+2G+4\alpha^2 GL  \tau_{\max}).$ 
Now setting $z=z^*$ in the above inequality, and noting that $\langle\Phi(z_{k}),z_{k}-z^*\rangle \geq 0$ based on Lemma \ref{lemma:nemirov}, we obtain the following recursive inequality that holds for all $k \in [T]$: 
\begin{equation}
    \|z_{k+1}-z^*\|^2 \leq \left(1+2\alpha^2 L G \tau_{\max}\right) \|z_{k}-z^*\|^2 + A.
\end{equation} 
Defining $r_k \triangleq \|z_{k}-z^*\|$, $\beta \triangleq \left(1+2\alpha^2 L G \tau_{\max}\right)$,  and iterating the above inequality, we obtain:
\begin{equation}
\begin{aligned}
    r^2_{k} & \leq  \beta^{k-1} r^2_1 + \left( \sum_{j=0}^{k-2} \beta^{j}\right) A \\
            & \leq \beta^{k-1} r^2_1 + \frac{\beta^k}{\beta-1} A \\
            & \leq \beta^T r^2_1 + \frac{\beta^T}{\beta-1} A. 
\end{aligned}   
\end{equation}
We will now bound each of the terms above by using the elementary fact that for any $c \in \mathbb{R}$, it holds that $(1+c) \leq e^c$. When the step-size $\alpha$ satisfies 
$$ \alpha \leq \frac{1}{2 \sqrt{LG \tau_{\max} T}}, $$
we have
\begin{equation}
    \beta^T \leq \left(1+\frac{1}{2T}\right)^T \leq e^{0.5} \leq 2. 
\end{equation}
Furthermore, it is easy to see that 
$$ \frac{A}{\beta-1} \leq \left(1+2G+\frac{1}{T}\right) \leq 4G.$$
Combining the above bounds leads to the claim of the lemma. This concludes the proof. 
\end{proof}

Based on the above result, let us introduce a set $\mathcal{H}$ as follows:
\begin{equation} 
\mathcal{H} \triangleq \{z | \Vert z- z^* \Vert^2 \leq 10 B\}, 
\label{eqn:set}
\end{equation}
where $B=\max\{\Vert z_1-z^* \Vert^2, G \}$. From Lemma~\ref{lemma: main bound}, we note that as long as the step-size $\alpha$ is chosen appropriately, the iterate sequence $\{z_k\}$ generated by \texttt{DGDA} belongs to $\mathcal{H}$. Moreover, $z^* \in \mathcal{H}$ trivially. With these observations in place, we now prove our main convergence result for \texttt{DGDA} for smooth convex-concave functions with bounded gradients. 

\begin{theorem} 
\label{thm:DGDACC}
Suppose Assumptions \ref{ass:boundedgrad} and \ref{ass:smoothness} hold $\forall x \in \mathbb{R}^m$ and $\forall y \in \mathbb{R}^n$. Let the step-size be chosen to satisfy
    $$\alpha=\frac{1}{2\sqrt{ LG     \tau_{\max} T}}.$$
    Then, the iterates generated by \texttt{DGDA} satisfy: 
\begin{align}
\max_{y: (\bar{x}_T,y) \in \mathcal{H}} f(\bar{x}_T,y) - \min_{x:(x, \bar{y}_T) \in \mathcal{H}} f(x, \bar{y}_T) \leq 44B \sqrt{\frac{GL\tau_{\max}}{T}}, 
\label{eqn:duality2} 
\nonumber
\end{align}
where $\bar{x}_T=(1/T)\sum_{k\in[T]} \hat{x}_k$, $\bar{y}_T=(1/T)\sum_{k\in[T]} \hat{y}_k$, and the set $\mathcal{H}$ is as defined in Eq.~\eqref{eqn:set}. 
\end{theorem}
\begin{proof} Recall the following notation from Lemma~\ref{lemma: main bound}: $r_k = \Vert z_k - z^* \Vert$ and $B=\max\{r^2_1, G\}.$ Let us start by noting that the choice of step-size in the statement of the theorem complies with that used to establish Lemma~\ref{lemma: main bound}. Thus, we can invoke Lemma~\ref{lemma: main bound} to conclude that for any $z\in \mathcal{H}$, the following is true:
\begin{equation}
\Vert z_k - z \Vert^2 \leq 2 r^2_k + 2 \Vert z - z^* \Vert^2 \leq 40B,
\label{eqn:set_bnd}
\end{equation}
where the last inequality follows from the definition of the set $\mathcal{H}$. 
Using Eq.~\eqref{eqn:key_ineq} from Lemma~\ref{lemma: main bound}, we then have for any $z\in \mathcal{H}$: 
 \begin{equation}
\begin{aligned}
2\alpha\langle\Phi(z_{k}),z_{k}-z\rangle & \leq \Vert z_k -z \Vert^2 - \Vert z_{k+1} -z \Vert^2 + A\\
 & \hspace{2mm} + 2\alpha^2 LG \tau_{\max} \Vert z_k -z \Vert^2 \\
& \leq \Vert z_k -z \Vert^2 - \Vert z_{k+1} -z \Vert^2 + \bar{A}, 
\end{aligned}
\nonumber
\end{equation}
where $\bar{A}=A+80 \alpha^2 LGB \tau_{\max}$, $A = 2 \alpha^2 G L \tau_{\max} (1+2G+4\alpha^2 GL  \tau_{\max})$, and we used Eq.~\eqref{eqn:set_bnd}. Now summing the above inequality from $k=1$ to $T$, we obtain
\begin{equation}
\begin{aligned}
\sum_{k=1}^{T}2\alpha\langle\Phi(z_{k}),z_{k}-z\rangle \leq & \Vert z_1 - z \Vert ^2 + \bar{A} T. 
\label{eqn:final_bnd1}
\end{aligned}
\end{equation}
Moreover, from Proposition 1  in \cite{aryan}, we have
\begin{equation}
\sum_{k=1}^{T}2\alpha\langle\Phi(z_{k}),z_{k}-z\rangle \geq 2\alpha T (f(\bar{x}_T,y)-f(x,\bar{y}_T)). 
\end{equation}
Combining the above display with Eq.~\eqref{eqn:final_bnd1} then yields the following bound $\forall z= [x;y] \in \mathcal{H}$: 
\begin{equation}
f(\bar{x}_T,y)-f(x,\bar{y}_T) \leq \frac{ \Vert z_1 - z \Vert^2}{2 \alpha T} + \frac{\bar{A}}{2\alpha}. 
\label{eqn:final_bnd2}
\end{equation}
Let us simplify the bound by first noting that for $\alpha$ chosen as in the statement of the theorem, it holds that $\bar{A} \leq 88 \alpha^2 GBL \tau_{\max}$. Moreover, since $z \in \mathcal{H}$, we have
$$ \Vert z_1 - z \Vert^2 \leq 2 r^2_1 + 2 \Vert z-z^* \Vert^2 \leq 22B.$$ 
Plugging in the above bounds in Eq.~\eqref{eqn:final_bnd2} then gives us:
\begin{equation}
f(\bar{x}_T,y)-f(x,\bar{y}_T) \leq \frac{11B}{\alpha T} + 44 \alpha G B L \tau_{\max}.
\end{equation}
The result follows from simply substituting the choice of $\alpha$ in the statement of the theorem. 
\end{proof}

\textbf{Discussion.} The main message conveyed by Theorem~\ref{thm:DGDACC} is that for smooth convex-concave functions with bounded gradients, the convergence rates of \texttt{DGDA} and \texttt{DEG} are identical in terms of their dependence on $\tau_{\max}$ and $T$. This complies with the intuition developed earlier in the paper that \texttt{EG} under delays behaves like \texttt{GDA}. 

\section{Analysis of Delayed Gradient Descent-Ascent for Strongly Convex-Strongly Concave functions}
\label{sec:DGDA_SCSC}
For smooth strongly convex-strongly concave functions, it is known that \texttt{GDA} guarantees linear convergence to the saddle point in the absence of delays~\cite{fallah2020optimal}. In this section, we ask: \textit{Does \texttt{DGDA} (Algorithm~\ref{algo:DGDA})  
also guarantee linear convergence to the saddle point for smooth strongly convex-strongly concave functions?} Our analysis in this section will provide an answer to this question in the affirmative. Moreover, we will precisely quantify how the maximum delay $\tau_{\max}$ slackens the exponent of linear convergence relative to when there is no delay. To get started, we now provide a formal definition of strongly convex-strongly concave functions. 

\begin{Assumption}
\label{ass:Sc-Sc} The function $f(x,y)$ is $\mu$-strongly convex-$\mu$-strongly concave (SC-SC) over $\mathcal{X} \times \mathcal{Y}$, i.e., for all $x_1,x_2\in \mathcal{X}$ and $y_1,y_2\in \mathcal{Y}$, the following holds:
\begin{equation}
\resizebox{1\hsize}{!}{$
\begin{aligned}
    f(x_2,y_1) &\geq f(x_1,y_1) +\langle \nabla_{x} f(x_1,y_1) , x_2-x_1 \rangle  + \frac{\mu}{2} \|x_2-x_1\|^2, \\
    f(x_1,y_2)& \leq f(x_1,y_1) +\langle \nabla_{y} f(x_1,y_1) , y_2-y_1 \rangle  - \frac{\mu}{2} \|y_2-y_1\|^2.
\end{aligned}
$}
\nonumber
\end{equation}
 \end{Assumption}

Throughout this section, we will set $\mathcal{X}=\mathbb{R}^m$ and $\mathcal{Y}=\mathbb{R}^n$, i.e., we will make no assumption of bounded domains. Furthermore, unlike prior sections, we will drop the assumption of bounded gradients, i.e., we will no longer work under Assumption~\ref{ass:boundedgrad}. 

\textbf{Analysis of \texttt{DGDA}.} To proceed, we start by recalling two results from \cite{fallah2020optimal} that will play a crucial role in our subsequent analysis; at this point, we remind the reader of the definition of $\Phi(\cdot)$ in Eq.~\eqref{eq:Phi}. 

\begin{lemma}[\cite{fallah2020optimal} ]\label{lemma:fallah1}
Suppose Assumptions \ref{ass:smoothness} and \ref{ass:Sc-Sc} hold. Then, $ \forall z, \hat{z} \in \mathbb{R}^{m+n}$, we have 
\begin{align}
L \|z - \hat{z}\|^2 \geq \langle\Phi(z) - \Phi(\hat{z}) ,z-\hat{z}\rangle \geq \mu \|z - \hat{z}\|^2.
\end{align}
\end{lemma}

\begin{lemma}[\cite{fallah2020optimal} ]\label{lemma:fallah2}
Suppose Assumptions \ref{ass:smoothness} and \ref{ass:Sc-Sc} hold. Then, $ \forall z, \hat{z} \in \mathbb{R}^{m+n}$, we have 
\begin{align}
\langle \Phi(z) - \Phi(\hat{z}), z-\hat{z}\rangle \geq \frac{\mu}{4L^2} \|\Phi(z) - \Phi(\hat{z})\|^2.
\end{align}
\end{lemma}

Recall the definitions of iterate-suboptimality and delay-induced error: $r_k=\|z_k-z^*\|$ and $e_k=\Phi(z_{k})-\Phi(z_{k-\tau_k})$. As before, our starting point will be to establish a bound on $\Vert e_k \Vert $. However, to establish a linear convergence rate, we need to provide a finer analysis relative to that in Lemmas~\ref{lemma:error} and \ref{lemma:DGDAerr}. In particular, unlike these results which established uniform convergence bounds on $\Vert e_k \Vert$, we will instead seek to bound $\Vert e_k \Vert$ as a function of a suitably defined iterate-suboptimality-metric. Our next result formalizes this idea. 
\begin{lemma}\label{lemma:T_1 bound}
Suppose Assumptions \ref{ass:smoothness} and \ref{ass:Sc-Sc} hold $ \forall z \in \mathbb{R}^{m+n}$. Then, for any $k\in [T]$, the delay-induced error $e_k = \Phi(z_k)-\Phi(z_{k-\tau_k})$ for \texttt{DGDA} satisfies
\begin{align}
\|e_k\|\leq 2 \alpha M_k,
\end{align}
where $M_k=L \tau_{\max} (\frac{4L^2}{\mu}+4L)\max_{k-2\tau_{\max}\leq t\leq k} r_t$.
\end{lemma}

\begin{proof}
    For bounding $e_k$, observe that 
    \begin{equation}
\begin{aligned}
    \|e_k\|=&\|\Phi(z_{k-\tau_k}) -\Phi(z_{k})\|\\& \overset{(a)} \leq 2L\|z_{k-\tau_k} -z_{k}\|\\&\leq 2L\sum_{j=k-\tau_k}^{k-1}\|z_{j+1} -z_{j}\|\\
    &\leq 2\alpha L\sum_{j=k-\tau_k}^{k-1}\|\Phi(z_{j-\tau_j})\| \\
    &\leq 2\alpha L\sum_{j=k-\tau_k}^{k-1}\left(\|\Phi(z_{j})\| +\|\Phi(z_{j-\tau_j})-\Phi(z_{j})\|\right)
    \\
    & \overset{(b)} \leq 2\alpha L\sum_{j=k-\tau_k}^{k-1}\left(\|\Phi(z_{j})\| +2L\|z_{j-\tau_j}-z_{j}\|\right)
    \\
    &\leq 2\alpha L\sum_{j=k-\tau_k}^{k-1}\left(\|\Phi(z_{j})\| +2Lr_{j-\tau_j}+ 2Lr_{j}\right).
\end{aligned}
\label{eqn:DGDA_error}
  \end{equation}
From Lemma \ref{lemma:nemirov}, we know that $\Phi(z^*)=0$. Furthermore, 
 from Lemma  \ref{lemma:fallah2} and the Cauchy–Schwarz inequality, we obtain
\begin{equation}
   \|\Phi(z_k)\|\|z_k-z^*\| \geq\langle \Phi(z_k),z_k-z^*\rangle\geq \frac{\mu}{4L^2}\|\Phi(z_k)\|^2,
\nonumber
\end{equation}
which means 
$$
\|\Phi(z_k)\|\leq \frac{4L^2}{\mu}\|z_k-z^*\|=\frac{4L^2}{\mu}r_k,
$$
 Combining the above display with Eq.~\eqref{eqn:DGDA_error}, we obtain 
\begin{equation}
\begin{aligned}
\|e_k\|&\leq 2\alpha L\sum_{j=k-\tau_{\max}}^{k-1}\left(\frac{4L^2}{\mu}r_j +2Lr_{j-\tau_j}+2Lr_{j}\right) 
\\&\leq 2\alpha L \tau_{\max} \left(\frac{4L^2}{\mu}+4L\right) \max_{k-2\tau_{\max}\leq t\leq k-1} r_t\\& \leq 2\alpha M_k.
\end{aligned}
\end{equation}
\end{proof}

We will also make use of the following key result. 

\begin{lemma}[\cite{gurbuzbalaban2017convergence}] \label{lemma:mert}
Suppose we have a sequence of non-negative real numbers, ${V_k}$, satisfying the inequality
\begin{align*}
V_{k+1} &\leq p V_k + q \max_{k-d(k) \leq \ell \leq k} V_\ell,
\end{align*}
for some non-negative constants $p$ and $q$, where $k \geq 0$ and $0 \leq d(k) \leq d_{\max}$ for some positive constant $d_{\max}$. If $p+q<1$, then we have
$$
V_k \leq r^k V_0, \hspace{1mm} \textrm{where} \hspace{1mm}  r = (p+q)^{1/(1+d_{\max})}. 
$$
\end{lemma}
We now prove our main result for \texttt{DGDA} for the class of smooth strongly convex-strongly concave (SC-SC) functions.

\begin{theorem} 
\label{thm:scsc} Suppose Assumptions~\ref{ass:smoothness} and \ref{ass:Sc-Sc} hold $\forall z \in \mathbb{R}^{m+n}$. Let the step-size be chosen to satisfy
$$ \alpha = \frac{\mu^3}{1536 L^6 \tau^2_{max}}.$$ Then, the iterates generated by $\texttt{DGDA}$ satisfy:
\begin{align}
     r_{k}\leq \left(1-\frac{\mu^4}{3072 L^6\tau_{\max}^2 }\right)^{\frac{k-1}{6\tau_{\max}}} r_1,
\end{align}
where $r_k = \Vert z_k - z^* \Vert.$ 
\end{theorem}
\begin{proof}
From the update rule of the \texttt{DGDA} algorithm and Lemma \ref{lemma:T_1 bound}, we have
\begin{equation}
\label{eqn:DGDA_finalbnd1}
\begin{aligned}
    &\|z_{k+1}-z^*\|^2-(1-\alpha \mu)\|z_k-z^*\|^2=\\&\alpha \mu \|z_k-z^*\|^2-2 \alpha \langle \Phi(z_{k-\tau_k}),z_k-z^*\rangle+\alpha^2 \|\Phi(z_{k-\tau_k})\|^2
    \\&\leq\alpha \mu \|z_k-z^*\|^2-2 \alpha \langle\Phi(z_{k}), z_k-z^* \rangle+2\alpha^2 \|\Phi(z_{k})\|^2
    \\& \hspace{2mm} + 2 \alpha \langle e_k,z_k-z^*\rangle +2\alpha^2 \| e_k \|^2\\&\leq\underbrace{\alpha \mu \|z_k-z^*\|^2-2 \alpha \langle\Phi(z_{k}),z_k-z^*\rangle+2\alpha^2 \|\Phi(z_{k})\|^2}_{f_k}
    \\& \hspace{2mm}+\underbrace{4 \alpha^2M_kr_k+8\alpha^4 M_k^2}_{p_k}. 
\end{aligned}
\end{equation}
From Lemmas \ref{lemma:fallah1} and \ref{lemma:fallah2}, we further know that 
\begin{equation}
   \langle \Phi(z_k),z_k-z^*\rangle\geq \mu\|z_k-z^*\|^2, \hspace{2mm} \textrm{and}
\nonumber
\end{equation}
\begin{equation}
  \langle \Phi(z_k),z_k-z^*\rangle\geq \frac{\mu}{4L^2}\|\Phi(z_k)\|^2 .
\nonumber
\end{equation}
When $\alpha \leq \frac{\mu}{8 L^2}$ - a requirement met by the choice of step-size in the statement of the theorem -  it is easy to verify that the above equations imply $f_k \leq 0$, where $f_k$ is as in Eq.~\eqref{eqn:DGDA_finalbnd1}. We also have 
$$
 p_k \leq 12 \alpha^2 M^2_k \leq \alpha^2 C \left (\max_{k-2\tau_{\max}\leq t\leq k} r_t^2 \right),
$$
where $C=768 \frac{L^6}{\mu^2} \tau^2_{\max}$, and we used $L \geq \mu$ for simplifications. From the above discussion, we conclude that 
\begin{equation}
r_{k+1}^2\leq (1-\alpha \mu)r_k^2 + \alpha^2 C \left(\max_{k-2\tau_{\max}\leq t\leq k} r_t^2 \right).
\nonumber
\end{equation}
From the choice of step-size in the statement of the theorem, it is easy to verify that 
$ 1-\alpha \mu + \alpha^2 C = 1-0.5 \alpha \mu < 1. $ 
Thus, we can immediately apply Lemma \ref{lemma:mert} to arrive at the desired conclusion. This concludes the proof. 
\end{proof}

\textbf{Discussion.} Theorem \ref{thm:scsc} reveals that for smooth SC-SC functions, \texttt{DGDA} guarantees \textit{linear convergence of the iterates to the saddle-point}. The result also clearly demonstrates how the exponent of convergence gets affected by $\tau_{\max}.$ 

\bibliographystyle{IEEEtran} 
\bibliography{refs}

\begin{thebibliography}{10}
\providecommand{\url}[1]{#1}
\csname url@samestyle\endcsname
\providecommand{\newblock}{\relax}
\providecommand{\bibinfo}[2]{#2}
\providecommand{\BIBentrySTDinterwordspacing}{\spaceskip=0pt\relax}
\providecommand{\BIBentryALTinterwordstretchfactor}{4}
\providecommand{\BIBentryALTinterwordspacing}{\spaceskip=\fontdimen2\font plus
\BIBentryALTinterwordstretchfactor\fontdimen3\font minus
  \fontdimen4\font\relax}
\providecommand{\BIBforeignlanguage}[2]{{%
\expandafter\ifx\csname l@#1\endcsname\relax
\typeout{** WARNING: IEEEtran.bst: No hyphenation pattern has been}%
\typeout{** loaded for the language `#1'. Using the pattern for}%
\typeout{** the default language instead.}%
\else
\language=\csname l@#1\endcsname
\fi
#2}}
\providecommand{\BIBdecl}{\relax}
\BIBdecl

\bibitem{von}
J.~Von~Neumann and O.~Morgenstern, \emph{Theory of games and economic
  behavior}.\hskip 1em plus 0.5em minus 0.4em\relax Princeton university press,
  2007.

\bibitem{goodfellow2020generative}
I.~Goodfellow, J.~Pouget-Abadie, M.~Mirza, B.~Xu, D.~Warde-Farley, S.~Ozair,
  A.~Courville, and Y.~Bengio, ``Generative adversarial networks,'' \emph{Comm.
  of the ACM}, vol.~63, no.~11, pp. 139--144, 2020.

\bibitem{ben2009robust}
A.~Ben-Tal, L.~El~Ghaoui, and A.~Nemirovski, \emph{Robust optimization}.\hskip
  1em plus 0.5em minus 0.4em\relax Princeton university press, 2009, vol.~28.

\bibitem{madry2017towards}
A.~Madry, A.~Makelov, L.~Schmidt, D.~Tsipras, and A.~Vladu, ``Towards deep
  learning models resistant to adversarial attacks,'' \emph{arXiv preprint
  arXiv:1706.06083}, 2017.

\bibitem{korpelevich}
G.~M. Korpelevich, ``The extragradient method for finding saddle points and
  other problems,'' \emph{Matecon}, vol.~12, pp. 747--756, 1976.

\bibitem{nedic}
A.~Nedi{\'c} and A.~Ozdaglar, ``Subgradient methods for saddle-point
  problems,'' \emph{Journal of optimization theory and applications}, vol. 142,
  no.~1, pp. 205--228, 2009.

\bibitem{daskalakis}
C.~Daskalakis, A.~Ilyas, V.~Syrgkanis, and H.~Zeng, ``Training gans with
  optimism,'' \emph{arXiv preprint arXiv:1711.00141}, 2017.

\bibitem{mokhtari2020unified}
A.~Mokhtari, A.~Ozdaglar, and S.~Pattathil, ``A unified analysis of
  extra-gradient and optimistic gradient methods for saddle point problems:
  Proximal point approach,'' in \emph{International Conference on Artificial
  Intelligence and Statistics}.\hskip 1em plus 0.5em minus 0.4em\relax PMLR,
  2020, pp. 1497--1507.

\bibitem{duchi2015asynchronous}
J.~C. Duchi, S.~Chaturapruek, and C.~R{\'e}, ``Asynchronous stochastic convex
  optimization,'' \emph{arXiv preprint arXiv:1508.00882}, 2015.

\bibitem{doan2017}
T.~T. Doan, C.~L. Beck, and R.~Srikant, ``On the convergence rate of
  distributed gradient methods for finite-sum optimization under communication
  delays,'' \emph{Proceedings of the ACM on Measurement and Analysis of
  Computing Systems}, vol.~1, no.~2, pp. 1--27, 2017.

\bibitem{arjevani2020tight}
Y.~Arjevani, O.~Shamir, and N.~Srebro, ``A tight convergence analysis for
  stochastic gradient descent with delayed updates,'' in \emph{Algorithmic
  Learning Theory}.\hskip 1em plus 0.5em minus 0.4em\relax PMLR, 2020, pp.
  111--132.

\bibitem{stich19}
S.~U. Stich and S.~P. Karimireddy, ``The error-feedback framework: Better rates
  for sgd with delayed gradients and compressed communication,'' \emph{arXiv
  preprint arXiv:1909.05350}, 2019.

\bibitem{koloskova2022}
A.~Koloskova, S.~U. Stich, and M.~Jaggi, ``Sharper convergence guarantees for
  asynchronous sgd for distributed and federated learning,'' \emph{Advances in
  Neural Information Processing Systems}, vol.~35, pp. 17\,202--17\,215, 2022.

\bibitem{adibi2022distributed}
A.~Adibi, A.~Mitra, G.~J. Pappas, and H.~Hassani, ``Distributed statistical
  min-max learning in the presence of byzantine agents,'' in \emph{Proc. of the
  61st IEEE Conference on Decision and Control}, 2022, pp. 4179--4184.

\bibitem{aryan}
A.~Mokhtari, A.~E. Ozdaglar, and S.~Pattathil, ``Convergence rate of $o(1/k)$
  for optimistic gradient and extragradient methods in smooth convex-concave
  saddle point problems,'' \emph{SIAM Journal on Optimization}, vol.~30, no.~4,
  pp. 3230--3251, 2020.

\bibitem{nemirov}
A.~Nemirovski, ``Prox-method with rate of convergence o (1/t) for variational
  inequalities with lipschitz continuous monotone operators and smooth
  convex-concave saddle point problems,'' \emph{SIAM Journal on Optimization},
  vol.~15, no.~1, pp. 229--251, 2004.

\bibitem{fallah2020optimal}
A.~Fallah, A.~Ozdaglar, and S.~Pattathil, ``An optimal multistage stochastic
  gradient method for minimax problems,'' in \emph{Proc. of the 59th IEEE
  Conference on Decision and Control}, 2020, pp. 3573--3579.

\bibitem{gurbuzbalaban2017convergence}
M.~Gurbuzbalaban, A.~Ozdaglar, and P.~A. Parrilo, ``On the convergence rate of
  incremental aggregated gradient algorithms,'' \emph{SIAM Journal on
  Optimization}, vol.~27, no.~2, pp. 1035--1048, 2017.

\end{thebibliography}
\end{document}